%% file: SubmodularStreaming.tex
\documentclass[10pt]{article}

\usepackage[paperwidth=7in, paperheight=10in, textwidth=5.25in, textheight=8.2in]{geometry}
\usepackage[ruled,vlined,linesnumbered]{algorithm2e}
\usepackage{subfigure}
\usepackage{amsmath}
\usepackage{amsthm}
\usepackage{amssymb}
\usepackage{natbib}
\usepackage{graphicx}
\usepackage{multirow}
\usepackage{array}
\usepackage{booktabs}
\usepackage{mathtools}
\usepackage{enumitem}
\usepackage{float}
\usepackage{nicefrac}
\usepackage{forloop}
\usepackage[hidelinks]{hyperref}
\usepackage{url}
\usepackage{authblk}
\usepackage{environ}

\newtheorem{theorem}{Theorem}

\newtheorem{observation}[theorem]{Observation}

\newtheorem{lemma}[theorem]{Lemma}
\newtheorem{corollary}[theorem]{Corollary}
\newtheorem{reduction}[theorem]{Reduction}
\newtheorem{proposition}[theorem]{Proposition}
\makeatletter
\newcommand{\newreptheorem}[2]{%
	\newenvironment{rep#1}[1]{%
		\expandafter\renewcommand\csname the#2\endcsname{\ref*{##1}}%
		\expandafter\renewcommand\csname theH#2\endcsname{repeat.##1}%
		\begin{#1}}%
		{\end{#1}%
		\addtocounter{#2}{-1}}}
\makeatother

\newreptheorem{theorem}{theorem}
\newreptheorem{observation}{theorem}
\newreptheorem{lemma}{theorem}
\newreptheorem{proposition}{theorem}

\newcommand{\eps}{\varepsilon}
\newcommand{\ie}{{i.e.}}
\newcommand{\ExchangeAlg}{{\textsc{Exchange-Candidate}}}
\newcommand{\characteristic}{{\mathbf{1}}}
\newcommand{\nnR}{{\bR_{\geq 0}}}
\newcommand{\AlgSeq}{\textsc{seqDPP}\xspace}
\newcommand{\AlgLocal}{\textsc{Local-Search}\xspace}
\newcommand{\AlgSampling}{\textsc{Sample-Streaming}\xspace}
\newcommand{\euler}{e}

\newcommand{\defcal}[1]{\expandafter\newcommand\csname c#1\endcsname{{\mathcal{#1}}}}
\newcommand{\defbb}[1]{\expandafter\newcommand\csname b#1\endcsname{{\mathbb{#1}}}}
\newcounter{calBbCounter}
\forLoop{1}{26}{calBbCounter}{
	\edef\letter{\Alph{calBbCounter}}
	\expandafter\defcal\letter
	\expandafter\defbb\letter
}

\title{Do Less, Get More: Streaming Submodular Maximization with Subsampling}
\usepackage{times}

\author[1]{Moran Feldman}
\author[2]{Amin Karbasi}
\author[2]{Ehsan Kazemi}
\affil[1]{Department of Mathematics and Computer Science, Open University of Israel}
\affil[2]{Yale Institute for Network Science\\Yale University}
\date{}

\begin{document}

\maketitle
\input{./tex/Abstract}

\thispagestyle{empty}
\pagenumbering{Alph}
\newpage 
\pagenumbering{arabic}

\input{./tex/Introduction}
\input{./tex/Preliminaries}
\input{./tex/Algorithm}
\input{./tex/Experiment}

\input{./tex/Conclusion}
\paragraph{Acknowledgements.} The work of Moran Feldman was supported in part by Israel Science Foundation (grant no. 1357/16). The work of Amin Karbasi was supported by DARPA Young Faculty Award (D16AP00046) and AFOSR Young Investigator Award (FA9550-18-1-0160). The work of Ehsan Kazemi was supported by Swiss National Science Foundation (Early Postdoc.Mobility) under grant
number 168574.

\bibliographystyle{plainnat}
\bibliography{./tex/SubmodularStreaming}

\end{document}

%% file: tex/Abstract.tex
\begin{abstract}
In this paper, we develop the first one-pass streaming algorithm for submodular maximization  that does not evaluate the entire stream even once. By carefully subsampling each element of data stream, our algorithm enjoys the tightest approximation guarantees in various settings while having the smallest memory footprint and requiring the lowest number of function evaluations. More specifically, for a monotone submodular function and a $p$-matchoid constraint, our randomized algorithm achieves a $4p$ approximation ratio (in expectation) with $O(k)$ memory and $O(km/p)$ queries per element ($k$ is the size of the largest feasible solution and $m$ is the number of matroids used to define the constraint). For the non-monotone case, our approximation ratio increases only slightly to $4p+2-o(1)$.  To the best or our knowledge, our algorithm is the first that combines the benefits of streaming and subsampling in a novel way in order to truly scale submodular maximization to massive machine learning problems. To showcase its practicality, we empirically evaluated the performance of our algorithm on a video summarization application and observed that it outperforms the state-of-the-art algorithm by up to fifty fold, while maintaining practically the same utility.
	
\end{abstract}

%% file: tex/Introduction.tex
\section{Introduction}
Submodularity  characterizes a  wide variety of discrete optimization problems that naturally occur in machine learning and artificial intelligence \citep{BB17}. Of particular interest is submodular maximization, which captures many novel instances of data summarization such as active set selection in non-parametric learning~\citep{MKSK16}, image summarization~\citep{TIWB14}, corpus summarization~\citep{LB11}, fMRI  parcellation \citep{SKSC17}, and removing redundant elements from DNA sequencing \citep{LBN18}, to name a few.  
%
%

Often the collection of elements to be summarized is generated continuously, and it is important to maintain at real time a summary of the part of the collection generated so far. For example, a surveillance camera generates a continuous stream of frames, and it is desirable to be able to quickly get at every given time point a short summary of the frames taken so far. The na\"{i}ve way to handle such a data summarization task is to store the entire set of generated elements, and then, upon request, use an appropriate offline submodular maximization algorithm to generate a summary out of the stored set. Unfortunately, this approach is usually not practical both because it requires the system to store the entire generated set of elements and because the generation of the summary from such a large amount of data can be very slow. These issues have motivated previous works to use streaming submodular maximization algorithms for data summarization tasks~\citep{GK10, BMKK14,MKK17}.

The first works (we are aware of) to consider a one-pass streaming algorithm for submodular maximization problems were the work of~\citet{BMKK14}, who described a $1/2$-approximation streaming algorithm for maximizing a monotone submodular function subject to a cardinality constraint, and the work of \citet{CK15} who gave a $4p$-approximation streming algorithm for maximizing such functions subject to the intersection of $p$ matroid constraints. The last result was later extended by~\citet{CGQ15} to $p$-matchoids constraints. For non-monotone submodular objectives, the first streaming result was obtained by \citet{BFS15}, who described a randomized streaming algorithm achieving $11.197$-approximation for the problem of maximizing a non-monotone submodular function subject to a single cardinality constraint. Then, \citet{CGQ15} described an algorithm of the same kind achieving $(5p + 2 + 1/p)/(1-\eps)$-approximation for the problem of maximizing a non-monotone submodular function subject to a $p$-matchoid constraint, and a deterministic streaming algorithm achieving $(9p + O(\sqrt{p}))/(1-\eps)$-approximation for the same problem.\footnote{The algorithms of~\cite{CGQ15} use an offline algorithm for the same problem in a black box fashion, and their approximation ratios depend on the offline algorithm used. The approximation ratios stated here assume the state-of-the-art offline algorithms of~\citep{FHK17} which were published only recently, and thus, they are better than the approximation ratios stated by~\cite{CGQ15}.} Finally, very recently, \citet{MJK17} came up with a different deterministic algorithm for the same problem achieving an approximation ratio of $4p + 4\sqrt{p} + 1$.


In the field of submodular optimization, it is customary to assume that the algorithm has access to the objective function and constraint through oracles. In particular, all the above algorithms assume access to a value oracle that given a set $S$ returns the value of the objective function for this set, and to an independence oracle that given a set $S$ and an input matroid answers whether $S$ is feasible or not in that matroid. 
%
%
%
 Given access to these oracles, the algorithms of~\citet{CK15} and~\citet{CGQ15} for monotone submodular objective functions are quite efficient, requiring only $O(k)$ memory ($k$ is the size of the largest feasible set) and using only $O(km)$ value and independence oracle queries for processing a single element of the stream ($m$ is a the number of matroids used to define the $p$-matchoid constraint). However, the algorithms developed for non-monotone submodular objectives are much less efficient (see Table~\ref{tbl:summary} for their exact parameters).

\begin{table}
\caption{Streaming algorithms for submodular maximization subject to a $p$-matchoid constraint.} \label{tbl:summary}
\begin{center}
\begin{tabular}{m{1.52cm}m{1.75cm}>{\centering\arraybackslash}m{1.5cm}>{\centering\arraybackslash}m{1.51cm}>{\centering\arraybackslash}m{1.76cm}m{2.4cm}}
\toprule
\textbf{\footnotesize Kind of} & \textbf{\footnotesize Objective} & \textbf{\footnotesize Approx.} & \textbf{\footnotesize Memory} & \textbf{\footnotesize Queries per} & \textbf{\footnotesize Reference}\\
\textbf{\footnotesize Algorithm} & \textbf{\footnotesize Function} & \textbf{\footnotesize Ratio} && \textbf{\footnotesize Element} & \\
\hline
{\footnotesize Deterministic} & {\footnotesize Monotone} & {\footnotesize $4p$} & {\footnotesize $O(k)$} & {\footnotesize $O(km)$} & {\footnotesize \citealp{CGQ15}}\\
{\footnotesize Randomized} & {\footnotesize Non-monotone} & {\footnotesize $\frac{5p + 2 + 1/p}{{1-\eps}}$} & {\footnotesize $O(\frac{k}{\eps^2} \log \frac{k}{\eps})$} & {\footnotesize $O(\frac{k^2m}{\eps^2} \log \frac{k}{\eps})$} & {\footnotesize \citealp{CGQ15}}\\
{\footnotesize Deterministic} & {\footnotesize Non-monotone} & {\footnotesize $\frac{9p + O(\sqrt{p})}{1-\eps}$} & {\footnotesize $O(\frac{k}{\eps}\log\frac{k}{\eps})$} & {\footnotesize $O(\frac{km}{\eps}\log \frac{k}{\eps})$} & {\footnotesize \citealp{CGQ15}}\\
{\footnotesize Deterministic} & {\footnotesize Non-monotone} & {\footnotesize $4p + 4\sqrt{p} + 1$} & {\footnotesize $O(k\sqrt{p})$} & {\footnotesize $O(\sqrt{p}km)$} & {\footnotesize \citealp{MJK17}\footnotemark}\\
\hline
{\footnotesize Randomized} & {\footnotesize Monotone} & {\footnotesize $4p$} & {\footnotesize $O(k)$} & {\footnotesize $O(km/p)$} & {\footnotesize This paper}\\
{\footnotesize Randomized} & {\footnotesize Non-monotone} & {\footnotesize $4p + 2 - o(1)$} & {\footnotesize $O(k)$} & {\footnotesize $O(km/p)$} & {\footnotesize This paper}\\
	\bottomrule
\end{tabular}
\end{center}
\end{table}
\footnotetext{The memory and query complexities of the algorithm of~\citet{MJK17} have been calculated based on the corresponding complexities of the algorithm of~\citep{CGQ15} for monotone objectives and the properties of the reduction used by~\citep{MJK17}. We note that these complexities do not match the memory and query complexities stated by~\citep{MJK17} for their algorithm.}

In this paper, we describe a new randomized streaming algorithm for maximizing a submodular function subject to a $p$-matchoid constraint. Our algorithm obtains an improved approximation ratio of $2p + 2\sqrt{p(p + 1)} + 1 = 4p + 2 - o(1)$, while using only $O(k)$ memory and $O(km/p)$ value and independence oracle queries (in expectation) per element of the stream, which is even less than the number of oracle queries used by the state-of-the-art algorithm for monotone submodular objectives. Moreover, when the objective function is monotone, our algorithm (with slightly different parameter values) achieves an improved approximation ratio of $4p$ using the same memory and oracle query complexities, \ie, it matches the state-of-the-art algorithm for monotone objectives in terms of the approximation ratio, while improving over it in terms of the number of value and independence oracle queries used. Additionally, we would like to point out that our algorithm also works in the online model with preemption suggested by~\citet{BFS15} for submodular maximization problems. Thus, our result for non-monotone submodular objectives represents the first non-trivial result in this model for such objectives for any constraint other than a single matroid constraint. For a single matroid constraint, an approximation ratio of $16$ (which improves to $8.734$ for cardinality constraints) was given by~\citet{CHJKT17}, and our algorithm improves it to $3 + 2\sqrt{2} \approx 5.828$ since a single matroid is equivalent to $1$-matchoid.

In addition to mathematically analyzing our algorithm, we also studied its practical performance in a video summarization task. 
We observed that, while our algorithm preserves the quality of the produced summaries, it outperforms the running time of the state-of-the-art algorithm by an order of magnitude. We also studied the effect of imposing different $p$-matchoid constraints on the video summarization.


\subsection{Additional Related Work}

The work on (offline) maximizing a monotone submodular function subject to a matroid constraint goes back to the classical result of \citet{FNW78}, who showed that the natural greedy algorithm gives an approximation ratio of $2$ for this problem. Later, an algorithm with an improved approximation ratio of $\euler / (\euler - 1)$ was found for this problem~\citep{CCPV11}, which is the best that can be done in polynomial time~\citep{NW78}. In contrast, the corresponding optimization problem for non-monotone submodular objectives is much less well understood.  After a long series of works~\citep{LMNS10,V13,GV11,FNS11,EN16}, the current best approximation ratio for this problem is $2.598$ \citep{BF16}, which is still far from the state-of-the-art inapproximability result of $2.093$ for this problem due to~\citep{GV11}.  

Several works have considered (offline) maximization of both monotone and non-monotone submodular functions subject to constraint families generalizing matroid constraints, including intersection of $p$-matroid constraints~\citep{LSV10}, $p$-exchange system constraints~\citep{FNSW11,W12}, $p$-extendible system constraints~\citep{FHK17} and $p$-systems constraints~\citep{FNW78,GRST10,MBK16,FHK17}. We note that the first of these families is a subset of the $p$-matchoid constraints studied by the current work, while the last two families generalize $p$-matchoid constraints. Moreover, the state-of-the-art approximation ratios for all these families of constraints are $p \pm O(\sqrt{p})$ both for monotone and non-monotone submodular objectives.

The study of submodular maximization in the streaming setting has been mostly surveyed above. However, we would like to note that besides the above mentioned results, there are also a few works on submodular maximization in the sliding window variant of the streaming setting~\citep{CNZ16,ELVZ17,WLT17}.

\subsection{Our Technique}

Technically, our algorithm is equivalent to dismissing every element of the stream with an appropriate probability, and then feeding the elements that have not been dismissed into the deterministic algorithm of~\citep{CGQ15} for maximizing a monotone submodular function subject to a $p$-matchoid constraint. The random dismissal of elements gives the algorithm two advantages. First, it makes it faster because there is no need to process the dismissed elements. Second, it is well known that such a dismissal often transforms an algorithm for monotone submodular objectives into an algorithm with some approximation guarantee also for non-monotone objectives. However, beside the above important advantages, dismissing elements at random also have an obvious drawback, namely, the dismissed elements are likely to include a significant fraction of the value of the optimal solution. The crux of the analysis of our algorithm is its ability to show that the above mentioned loss of value due to the random dismissal of elements does not affect the approximation ratio. To do so, we prove a stronger version of a structural lemma regarding graphs and matroids that was implicitly proved by~\citep{V11} and later stated explicitly by~\citep{CGQ15}. The stronger version we prove translates into an improvement in the bound on the performance of the algorithm, which is not sufficient to improve the guaranteed approximation ratio, but fortunately, is good enough to counterbalance the loss due to the random dismissal of elements.

We would like to note that the general technique of dismissing elements at random, and then running an algorithm for monotone submodular objectives on the remaining elements, was previously used by~\citep{FHK17} in the context of offline algorithms. However, the method we use in this work to counterbalance the loss of value due to the random dismissal of \textit{streaming} elements is completely unrelated to the way this was achieved in~\citep{FHK17}.

%% file: tex/Preliminaries.tex
\section{Preliminaries}

In this section, we introduce some notation and definitions that we later use to formally state our results. A set function $f\colon 2^\cN \to \bR$ on a ground set $\cN$ is \emph{non-negative} if $f(S) \geq 0$ for every $S \subseteq \cN$, \emph{monotone} if $f(S) \leq f(T)$ for every $S \subseteq T \subseteq \cN$ and \emph{submodular} if $f(S) + f(T) \geq f(S \cup T) + f(S \cap T)$ for every $S,T \subseteq \cN$. Intuitively, a submodular function is a function that obeys the property of diminishing returns, \ie, the marginal contribution of adding an element to a set diminishes as the set becomes larger and larger. Unfortunately, it is somewhat difficult to relate this intuition to the above (quite cryptic) definition of submodularity, and therefore, a more friendly equivalent definition of submodularity is often used. However, to present this equivalent definition in a simple form, we need some notation. Given a set $S$ and an element $u$, we denote by $S + u$ and $S - u$ the union $S \cup \{u\}$ and the expression $S \setminus \{u\}$, respectively. Additionally, the marginal contribution of $u$ to the set $S$ under the set function $f$ is written as $f(u \mid S) \triangleq f(S + u) - f(S)$. Using this notation, we can now state the above mentioned equivalent definition of submodularity, which is that a set function $f$ is submodular if and only if
\[
	f(u \mid S) \geq f(u \mid T) \quad \forall\; S \subseteq T \subseteq \cN \text{ and } u \in \cN \setminus T
	\enspace.
\]
Occasionally, we also refer to the marginal contribution of a set $T$ to a set $S$ (under a set function $f$), which we write as $f(T \mid S) \triangleq f(S \cup T) - f(S)$.

A \emph{set system} is a pair $(\cN, \cI)$, where $\cN$ is the ground set of the set system and $\cI \subseteq 2^\cN$ is the set of \emph{independent} sets of the set system. A \emph{matroid} is a set system which obeys three properties: (i) the empty set is independent, (ii) if $S \subseteq T \subseteq \cN$ and $T$ is independent, then so is $S$, and finally, (iii) if $S$ and $T$ are two independent sets obeying $|S| < |T|$, then there exists an element $u \in T \setminus S$ such that $S + u$ is independent. In the following lines we define two matroid related terms that we use often in our proofs, however, readers who are not familiar with matroid theory should consider reading a more extensive presentation of matroids, such as the one given by~\citep[Volume B]{S03}. A \emph{cycle} of a matroid is an inclusion-wise minimal dependent set, and an element $u$ is \emph{spanned} by a set $S$ if the maximum size independent subsets of $S$ and $S+u$ are of the same size. Note that it follows from these definitions that every element $u$ of a cycle $C$ is spanned by $C - u$.

A set system $(\cN, \cI)$ is a \emph{$p$-matchoid}, for some positive integer $p$, if there exist $m$ matroids $(\cN_1, \cI_1), (\cN_2, \cI_2), \dotsc, (\cN_m, \cI_m)$ such that every element of $\cN$ appears in the ground set of at most $p$ out of these matroids and $\cI = \{S \subseteq 2^\cN \mid \forall_{1 \leq i \leq m}\; S \cap \cN_i \in \cI_i\}$. A simple example for a $2$-matchoid is $b$-matching. Recall that a set $E$ of edges of a graph is a $b$-matching if and only if every vertex $v$ of the graph is hit by at most $b(v)$ edges of $E$, where $b$ is a function assigning integer values to the vertices. The corresponding $2$-matchoid $\cM$ has the set of edges of the graph as its ground set and a matroid for every vertex of the graph, where the matroid $\cM_v$ of a vertex $v$ of the graph has in its ground set only the edges hitting $v$ and a set $E$ of edges is independent in $\cM_v$ if and only if $|E| \leq b(v)$. Since every edge hits only two vertices, it appears in the ground sets of only two vertex matroids, and thus, $\cM$ is indeed a $2$-matchoid. Moreover, one can verify that a set of edges is independent in $\cM$ if and only if it is a valid $b$-matching.

The problem of maximizing a set function $f\colon 2^\cN \to \bR$ subject to a $p$-matchoid constraint $\cM = (\cN, \cI)$ asks us to find an independent set $S \in \cI$ maximizing $f(S)$. In the streaming setting we assume that the elements of $\cN$ arrive sequentially in some adversarially chosen order, and the algorithm learns about each element only when it arrives. The objective of an algorithm in this setting is to maintain a set $S \in \cI$ which approximately maximizes $f$, and to do so with as little memory as possible. In particular, we are interested in algorithms whose memory requirement does not depend on the size of the ground set $\cN$, which means that they cannot keep in their memory all the elements that have arrived so far. Our two results for this setting are given by the following theorems. Recall that $k$ is the size of the largest independent set and $m$ is the number of matroids used to define the $p$-matchoid constraint.

\begin{theorem} \label{thm:monotone}
There is a streaming $4p$-approximation algorithm for the problem of maximizing a non-negative \emph{monotone} submodular function $f$ subject to a $p$-matchoid constraint whose space complexity is $O(k)$. Moreover, in expectation, this algorithm uses $O(km/p)$ value and independence oracle queries when processing each arriving element.
\end{theorem}

\newcommand{\thmNonMonotone}{%
There is a streaming $(2p + 2\sqrt{p(p+1} + 1)$-approximation algorithm for the problem of maximizing a non-negative submodular function $f$ subject to a $p$-matchoid constraint whose space complexity is $O(k)$. Moreover, in expectation, this algorithm uses $O(km/p)$ value and indpenence oracle queries when processing each arriving element.}
\begin{theorem} \label{thm:non_monotone}
\thmNonMonotone
\end{theorem}

%% file: tex/Algorithm.tex
\section{Algorithm} \label{sec:algorithm}

In this section we prove Theorems~\ref{thm:monotone} and~\ref{thm:non_monotone}. Throughout this section we assume that $f$ is a non-negative submodular function over the ground set $\cN$, and $\cM = (\cN, \cI)$ is a $p$-matchoid over the same ground set which is defined by the matroids $(\cN_1, \cI_1), (\cN_2, \cI_2), \dotsc, (\cN_m, \cI_m)$. Additionally, we denote by $u_1, u_2, \dotsc, u_n$ the elements of $\cN$ in the order in which they arrive. Finally, for an element $u_i \in N$ and sets $S, T \subseteq \cN$, we use the shorthands $f(u_i : S) = f(u_i \mid S \cap \{u_1, u_2, \dotsc, u_{i-1}\})$ and $f(T : S) = \sum_{u \in T} f(u : S)$. Intuitively, $f(u : S)$ is the marginal contribution of $u$ to the part of $S$ that arrived before $u$ itself. One useful property of this shorthand is given by the following observation. 

\newcommand{\obsTechnical}{For every two sets $S, T \subseteq \cN$, $f(T \mid S \setminus T) \leq f(T : S)$.}
\begin{observation} \label{obs:technical}
\obsTechnical
\end{observation}

\begin{proof}
	Let us denote the elements of $T$ by $u_{i_1}, u_{i_2}, \dotsc, u_{i_{|T|}}$, where $i_1 < i_2 < \dotsb < i_{|T|}$. Then,
	\begin{align*}
	f(T \mid S \setminus T)
	={} &
	\sum_{j = 1}^{|T|} f(u_{i_j} \mid (S \cup T) \setminus \{u_{i_j}, u_{i_{j + 1}} \dotsc, u_{i_{|T|}}\})\\
	\leq{} &
	\sum_{j = 1}^{|T|} f(u_{i_j} \mid S \setminus \{u_{i_j}, u_{i_j + 1} \dotsc, u_n\})\\
	={} &
	\sum_{j = 1}^{|T|} f(u_{i_j} \mid S \cap \{u_1, u_2, \dotsc, u_{i_j - 1}\})
	=
	\sum_{j = 1}^{|T|} f(u_{i_j} : S)
	=
	f(T : S)
	\enspace,
	\end{align*}
	where the inequality follows from the submodularity of $f$.
\end{proof}

Let us now present the algorithm we us to prove our results. This algorithm uses a procedure named {\ExchangeAlg} which appeared also in previous works, sometimes under the exact same name. {\ExchangeAlg} gets an independent set $S$ and an element $u$, and its role is to output a set $U \subseteq S$ such that $S \setminus U + u$ is independent. The pseudocode of {\ExchangeAlg} is given as Algorithm~\ref{alg:exchange_alg}.
\begin{algorithm}
\caption{{\ExchangeAlg} $(S, u)$}\label{alg:exchange_alg}
Let $U \gets \varnothing$.\\
\For{$\ell = 1$ \KwTo $m$}
{
	\If{$(S + u) \cap \cN_\ell \not \in \cI_\ell$}
	{
		Let $X_\ell \gets \{x \in S \mid ((S - x + u) \cap \cN_\ell) \in \cI_\ell\}$.\\
		Let $x_\ell \gets \arg \min_{x \in X_\ell} f(x : S)$.\\
		Update $U \gets U + x_\ell$.
	}
}
\Return{U}.
\end{algorithm}

Using the procedure {\ExchangeAlg}, we can now write our own algorithm, which is given as Algorithm~\ref{alg:actual}. This algorithm has two parameters, a probability $q$ and a value $c > 0$. Whenever the algorithm gets a new element $u$, it dismisses it with probability $1 - q$. Otherwise, the algorithm finds using {\ExchangeAlg} a set $U$ of elements whose removal from the current solution maintained by the algorithm allows the addition of $u$ to this solution. If the marginal contribution of adding $u$ to the solution is large enough compared to the value of the elements of $U$, then $u$ is added to the solution and the elements of $U$ are removed. While reading the pseudocode of the algorithm, keep in mind that $S_i$ represents the solution of the algorithm after $i$ elements have been processed.
\SetKwIF{With}{OtherwiseWith}{Otherwise}{with}{do}{otherwise with}{otherwise}{end}
\begin{algorithm}
\caption{\AlgSampling: Streaming Algorithm for a $p$-Matchoid Constraint} \label{alg:actual}
\DontPrintSemicolon
Let $S_0 \gets \varnothing$.\\
\For{every arriving element $u_i$}
{
	Let $S_i \gets S_{i-1}$.\\
	\With{probability $q$}
	{
		Let $U_i \gets {\ExchangeAlg}(S_{i-1}, u_i)$.\\
		\lIf{$f(u_i \mid S_{i-1}) \geq (1 + c) \cdot f(U_i : S_{i-1})$}
		{
			Let $S_i \gets S_{i-1} \setminus U_i + u_i$.
		}
	}
}
\Return{$S_n$}.
\end{algorithm}

\newcommand{\obsComplexity}{Algorithm~\ref{alg:actual} can be implemented using $O(k)$ memory and, in expectation, $O(qkm)$ value and independence oracle queries per arriving element.}
\begin{observation} \label{obs:complexity}
\obsComplexity
\end{observation}

\begin{proof}
	An implementation of Algorithm~\ref{alg:actual} has to keep in memory at every given time point only three sets: $S_i$, $U_i$ and $X_\ell$. Since these sets are all subsets of independent sets, each one of them contains at most $k$ elements, and thus, $O(k)$ memory suffices for the algorithm.
	
	An arriving element which is dismissed immediately (which happens with probability $1 - q$) does not require any value and independence oracle queries. The remaining elements require $O(km)$ such queries, and thus, in expectation an arriving element requires $q \cdot O(km) = O(qkm)$ oracle queries.
\end{proof}

Algorithm~\ref{alg:actual} adds an element $u_i$ to its solution if two things happen: (i) $u_i$ is not dismissed due to the random decision and (ii) the marginal contribution of $u_i$ with respect to the current solution is large enough compared to the value of $U_i$. Since checking (ii) requires more resources then checking (i), the algorithm checks (i) first. However, for analyzing the approximation ratio of Algorithm~\ref{alg:actual}, it is useful to assume that (ii) is checked first. Moreover, for the same purpose, it is also useful to assume that the elements that pass (ii) but fail (i) are added to a set $R$. The algorithm obtained after making these changes is given as Algorithm~\ref{alg:analysis}. One should note that this algorithm has the same output distribution as Algorithm~\ref{alg:actual}, and thus, the approximation ratio we prove for the first algorithm applies to the second one as well.
\begin{algorithm}
\caption{Streaming Algorithm for a $p$-Matchoid Constraint (Analysis Version)} \label{alg:analysis}
\DontPrintSemicolon
Let $S_0 \gets \varnothing$ and $R \gets \varnothing$.\\
\For{every arriving element $u_i$}
{
	Let $S_i \gets S_{i-1}$.\\
	Let $U_i \gets {\ExchangeAlg}(S_{i-1}, u_i)$.\\
	\If{$f(u_i \mid S_{i-1}) \geq (1 + c) \cdot f(U_i : S_{i-1})$}
	{
		\lWith{probability $q$}
		{
			Let $S_i \gets S_{i-1} \setminus U_i + u_i$.
		}
		\lOtherwise
		{
			Update $R \gets R + u_i$.
		}
	}
}
\Return{$S_n$}.
\end{algorithm}

Let us denote by $A$ the set of elements that ever appeared in the solution maintained by Algorithm~\ref{alg:analysis}---formally, $A = \bigcup_{i=1}^n S_i$. The following lemma and corollary show that the elements of $A \setminus S_n$ cannot contribute much to the output solution $S_n$ of Algorithm~\ref{alg:analysis}, and thus, their absence from $S_n$ does not make $S_n$ much less valuable than $A$.
\begin{lemma} \label{lem:marginals_sum}
$f(A \setminus S_n : S_n) \leq \frac{f(S_n)}{c}$.
\end{lemma}
\begin{proof}
Fix an element $u_i \in A$, then
\begin{align} \label{eq:single_element}
	f(S_i) - f(S_{i-1})
	={} &
	f(S_{i-1} \setminus U_i + u_i) - f(S_{i-1})
	=
	f(u_i \mid S_{i-1} \setminus U_i) - f(U_i \mid S_{i-1} \setminus U_i)\\\nonumber
	\geq{} &
	f(u_i \mid S_{i-1}) - f(U_i : S_{i-1})
	\geq
	c \cdot f(U_i : S_{i-1}) 
	\enspace,
\end{align}
where the first inequality follows from the submodularity of $f$ and Observation~\ref{obs:technical}, and the second inequality holds since the fact that Algorithm~\ref{alg:analysis} accepted $u_i$ into its solution implies $f(u_i \mid S_{i-1}) \geq (1 + c) \cdot f(U_i : S_{i-1})$.

We now observe that every element of $A \setminus S_n$ must have been removed exactly once from the solution of Algorithm~\ref{alg:analysis}, which implies that $\{U_i \mid u_i \in A\}$ is a disjoint partition of $A \setminus S_n$. Using this observation, we get
\[
	f(A \setminus S_n : S_n)
	=
	\sum_{u_i \in A} f(U_i : S_n)
	\leq
	\sum_{u_i \in A} \frac{f(S_i) - f(S_{i-1})}{c}
	=
	\frac{f(S_n) - f(\varnothing)}{c}
	\leq
	\frac{f(S_n)}{c}
	\enspace,
\]
where the first inequality follows from Inequality~\eqref{eq:single_element}, the second equality holds since $S_i = S_{i-1}$ whenever $u_i \not \in A$ and the second inequality follows from the non-negativity of $f$.
\end{proof}
\begin{corollary} \label{cor:sets_ratio}
$f(A) \leq \frac{c + 1}{c} \cdot f(S_n)$.
\end{corollary}
\begin{proof}
Since $S_n \subseteq A$ by definition,
\begin{align*}
	f(A)
	={} &
	f(A \setminus S_n \mid S_n) + f(S_n)
	\leq
	f(A \setminus S_n : S_n) + f(S_n)\\
	\leq{} &
	\frac{f(S_n)}{c} + f(S_n)
	=
	\frac{c + 1}{c} \cdot f(S_n)
	\enspace,
\end{align*}
where the first inequality follows from Observation~\ref{obs:technical} and the second from Lemma~\ref{lem:marginals_sum}.
\end{proof}

Our next goal is to show that the value of the elements of the optimal solution that do not belong to $A$ is not too large compared to the value of $A$ itself. To do so, we need a mapping from the elements of the optimal solution to elements of $A$. Such a mapping is given by Proposition~\ref{prop:mapping}. However, before we get to this proposition, let us first present Reduction~\ref{red:exact}, which simplifies Proposition~\ref{prop:mapping}.
\begin{reduction} \label{red:exact}
For the sake of analyzing the approximation ratio of Algorithm~\ref{alg:analysis}, one may assume that every element $u \in \cN$ belongs to \emph{exactly} $p$ out of the $m$ ground sets $\cN_1, \cN_2, \dotsc, \cN_m$ of the matroids defining $\cM$.
\end{reduction}
\begin{proof}
For every element $u \in \cN$ that belongs to the ground sets of only $p' < p$ out of the $m$ matroids $(\cN_1, \cN_1), (\cN_2, \cN_2), \dotsc, (\cN_m, \cI_m)$, we can add $u$ to $p - p'$ additional matroids as a free element (\ie, an element whose addition to an independent set always keeps the set independent). On can observe that the addition of $u$ to these matroids does not affect the behavior of Algorithm~\ref{alg:analysis} at all, but makes $u$ obey the technical property of belonging to \emph{exactly} $p$ out of the ground sets $\cN_1, \cN_2, \dotsc, \cN_m$.
\end{proof}

From this point on we implicitly make the assumption allowed by Reduction~\ref{red:exact}. In particular, the proof of Proposition~\ref{prop:mapping} relies on this assumption. 
\newcommand{\propMapping}{%
	For every set $T \in \cI$ which does not include elements of $R$, there exists a mapping $\phi_T$ from elements of $T$ to multi-subsets of $A$ such that
	\begin{itemize}
		\item every element $u \in S_n$ appears at most $p$ times in the multi-sets of $\{\phi_T(u) \mid u \in T\}$.
		\item every element $u \in A \setminus S_n$ appears at most $p - 1$ times in the multi-sets of $\{\phi_T(u) \mid u \in T\}$.
		\item every element $u_i \in T \setminus A$ obeys
		$
		f(u_i \mid S_{i-1})
		\leq
		(1 + c) \cdot \sum_{u_j \in \phi_T(u_i)} f(u_j : S_{d(j) - 1})
		$.
		\item every element $u_i \in T \cap A$ obeys
		$
		f(u_i \mid S_{i-1})
		\leq
		f(u_j : S_{d(j)-1})
		$
		for every
		$
		u_j \in \phi_T(u_i)
		$,
		and the multi-set $\phi_T(u_i)$ contains exactly $p$ elements (including repetitions).
\end{itemize}}
\begin{proposition} \label{prop:mapping}
	\propMapping
\end{proposition}

The proof of Proposition~\ref{prop:mapping} is quite long and involves many details, and thus, we defer it to Section~\ref{ssc:mapping}. Instead, let us prove now a very useful technical observation. To present this observation we need some additional definitions. Let $Z = \{u_i \in \cN \mid f(u_i \mid S_{i-1}) < 0\}$. Additionally, for every $1 \leq i \leq n$, we define
\[
	d(i)
	=
	\begin{cases}
		1 + \max \{i \leq j \leq n \mid u_i \in S_j\} & \text{if $u_i \in A$} \enspace,\\
		i & \text{otherwise} \enspace.
	\end{cases}
\]
In general, $d(i)$ is the index of the element whose arrival made Algorithm~\ref{alg:analysis} remove $u_i$ from its solution. Two exceptions to this rule are as follows. If $u_i$ was never added to the solution, then $d(i) = i$; and if $u_i$ was never removed from the solution, then $d(i) = n + 1$. 
\newcommand{\obsSetProperties}{Consider an arbitrary element $u_i \in \cN$.
\begin{itemize}
	\item If $u_i \not \in Z$, then $f(u_i : S_{i'}) \geq 0$ for every $i' \geq i - 1$. In particular, since $d(i) \geq i$, $f(u_i : S_{d(i) - 1}) \geq 0$
	\item $A \cap (R \cup Z) = \varnothing$.
\end{itemize}}
\begin{observation} \label{obs:set_properties}
\obsSetProperties
\end{observation}
\begin{proof}
	%
	To see why the first part of the observation is true, consider an arbitrary element $u_i \not \in Z$. Then,
	\[
	0
	\leq
	f(u_i \mid S_{i - 1})
	\leq
	f(u \mid S_{i'} \cap \{u_1, u_2, \dotsc, u_{i - 1}\})
	=
	f(u : S_{i'})
	\enspace,
	\]
	where the second inequality follows from the submodularity of $f$ and the inclusion $S_{i'} \cap \{u_1, u_2, \dotsc,\allowbreak u_{i - 1}\} \subseteq S_{i-1}$ (which holds because elements are only added by Algorithm~\ref{alg:analysis} to its solution at the time of their arrival).
	
	It remains to prove the second part of the observation. Note that Algorithm~\ref{alg:analysis} adds every arriving element to at most one of the sets $A$ and $R$, and thus, these sets are disjoint; hence, to prove the observation it is enough to show that $A$ and $Z$ are also disjoint. Assume towards a contradiction that this is not the case, and let $u_i$ be the first element to arrive which belongs to both $A$ and $Z$. Then,
	\[
	f(u_i \mid S_{i - 1})
	\geq
	(1 + c) \cdot f(U_i : S_{i - 1})
	=
	(1 + c) \cdot \sum_{u_j \in U_i} f(u_j : S_{d(j) - 1})
	\enspace.
	\]
	To see why that inequality leads to a contradiction, notice its leftmost hand side is negative by our assumption that $u_i \in Z$, while its rightmost hand side is non-negative by the first part of this observation since the choice of $u_i$ implies that no element of $U_i \subseteq S_{i - 1} \subseteq A \cap \{u_1, u_2, \dotsc, u_{i-1}\}$ can belong to $Z$.
\end{proof}

%
%

Using all the tools we have seen so far, we are now ready to prove the following theorem. Let $OPT$ be an independent set of $\cM$ maximizing $f$.
\begin{theorem} \label{thm:raw}
Assuming $q^{-1} = (1 + c)p + 1$, $\bE[f(S_n)] \geq \frac{c}{(1+c)^2p} \cdot \bE[f(A \cup OPT)]$.
\end{theorem}
\begin{proof}
Since $S_i \subseteq A$ for every $0 \leq i \leq n$, the submodularity of $f$ guarantees that
\begin{align*}
	f(A \cup OPT)
	\leq{} &
	f(A) + \sum_{u_i \in OPT \setminus (R \cup A)}  \mspace{-36mu} f(u_i \mid A) + \sum_{u_i \in (OPT \setminus A) \cap R} \mspace{-18mu} f(u_i \mid A) \\
	\leq{} &
	f(A) + \sum_{u_i \in OPT \setminus (R \cup A)}  \mspace{-36mu} f(u_i \mid S_{i-1}) + \sum_{u_i \in (OPT \setminus A) \cap R} \mspace{-18mu} f(u_i \mid S_{i - 1}) \\
	\leq{} &
	\frac{1+c}{c} \cdot f(S_n) + \sum_{u_i \in OPT \setminus (R \cup A)} \mspace{-36mu} f(u_i \mid S_{i-1})  + \sum_{u_i \in OPT \cap R} \mspace{-18mu} f(u_i \mid S_{i-1})
	\enspace,
\end{align*}
where the third inequality follows from Corollary~\ref{cor:sets_ratio} and the fact that $A \cap R = \varnothing$ by Observation~\ref{obs:set_properties}. Let us now consider the function $\phi_{OPT \setminus R}$ whose existence is guaranteed by Proposition~\ref{prop:mapping} when we choose $T = OPT \setminus R$. Then, the property guaranteed by Proposition~\ref{prop:mapping} for elements of $T \setminus A$ implies
\[
	\sum_{u_i \in OPT \setminus (R \cup A)} \mspace{-27mu} f(u_i \mid S_{i-1})
	\leq
	(1 + c ) \cdot \mspace{-18mu} \sum_{\substack{u_i \in OPT \setminus (R \cup A)\\u_j \in \phi_{OPT \setminus R}(u_i)}} \mspace{-27mu} f(u_j:S_{d(j) - 1} )
	\enspace.
\]
Additionally,
\begin{align*}
	&
	\sum_{\substack{u_i \in OPT \setminus (R \cup A)\\u_j \in \phi_{OPT \setminus R}(u_i)}} \mspace{-36mu} f(u_j:S_{d(j) - 1} ) + p \cdot \mspace{-18mu} \sum_{u_i \in OPT \cap A} \mspace{-18mu} f(u_i \mid S_{i-1})
	\leq
	\mspace{-18mu} \sum_{\substack{u_i \in OPT \setminus R\\u_j \in \phi_{OPT \setminus R}(u_i)}} \mspace{-36mu} f(u_j: S_{d(j)- 1} )\\
	\leq{} &
	p \cdot \sum_{u_j \in S_n} f(u_j:S_n) + (p-1) \cdot \mspace{-9mu} \sum_{u_j \in A \setminus S_n} \mspace{-9mu} f(u_j: S_{d(j) - 1})\\
	\leq{} &
	p \cdot f(S_n) + \frac{p-1}{c} \cdot f(S_n)
	=
	\frac{(1 + c) \cdot p - 1}{c}  \cdot f(S_n)
	\enspace,
\end{align*}
where the first inequality follows from the properties guaranteed by Proposition~\ref{prop:mapping} for elements of $T \cap A$ (note that the sets $OPT \setminus (R \cup A)$ and $OPT \cap A$ are a disjoint partition of $OPT \setminus R$ by Observation~\ref{obs:set_properties}) and the second inequality follows from the properties guaranteed by Proposition~\ref{prop:mapping} for elements of $A \setminus S_n$ and $S_n$ because every element $u_i$ in the multisets produced by $\phi_{OPT \setminus R}$ belongs to $A$, and thus, obeys $f(u_i : S_{d(i) - 1}) \geq 0$ by Observation~\ref{obs:set_properties}. Finally, the last inequality follows from Lemma~\ref{lem:marginals_sum} and the fact that $f(u_j : S_{d(j) - 1}) \leq f(u_j : S_n)$ for every $1 \leq j \leq n$. Combining all the above inequalities, we get
\begin{align} \label{eq:cup_bound}
	f(A&{} \cup OPT)
	\leq
	\frac{1 + c}{c} \cdot f(S_n) + \nonumber \\
	& (1 + c) \cdot \left[  \frac{(1 + c) \cdot p - 1}{c}  \cdot f(S_n) - p \cdot \mspace{-9mu} \sum_{u_i \in OPT \cap A} \mspace{-18mu} f(u_i \mid S_{i-1}) \right] + \sum_{u_i \in OPT \cap R} \mspace{-18mu} f(u_i \mid S_{i-1}) \nonumber \\
	={} &
	\frac{(1 + c)^2\cdot p}{c} \cdot f(S_n) - (1 + c)p \cdot \mspace{-9mu} \sum_{u_i \in OPT \cap A} \mspace{-18mu} f(u_i \mid S_{i-1}) + \sum_{u_i \in OPT \cap R} \mspace{-18mu} f(u_i \mid S_{i-1})
	\enspace.
\end{align}

By the linearity of expectation, to prove the theorem it only remains to show that the expectations of the last two terms on the rightmost hand side of Inequality~\eqref{eq:cup_bound} are equal. This is our objective in the rest of this proof. Consider an arbitrary element $u_i \in OPT$. When $u_i$ arrives, one of two things happens. The first option is that Algorithm~\ref{alg:analysis} discards $u_i$ without adding it to either its solution or to $R$. The other option is that Algorithm~\ref{alg:analysis} adds $u_i$ to its solution (and thus, to $A$) with probability $q$, and to $R$ with probability $1 - q$. The crucial observation here is that at the time of $u_i$'s arrival the set $S_{i - 1}$ is already determined, and thus, this set is independent of the decision of the algorithm to add $u$ to $A$ or to $R$; which implies the following equality (given an event $\cE$, we use here $\characteristic[\cE]$ to denote an indicator for it).
\[
	\frac{\bE[\characteristic[u_i \in A] \cdot f(u_i \mid S_{i - 1})]}{q}
	=
	\frac{\bE[\characteristic[u_i \in R] \cdot f(u_i \mid S_{i - 1})]}{1 - q}
	\enspace.
\]
Rearranging the last equality, and summing it up over all elements $u_i \in OPT$, we get
\[
	\frac{1 - q}{q} \cdot \bE\left[\sum_{u_i \in OPT \cap A_n} \mspace{-27mu} f(u_i \mid S_{i - 1})\right]
	=
	\bE\left[\sum_{u_i \in OPT \cap R} \mspace{-18mu} f(u_i \mid S_{i - 1})\right]
	\enspace.
\]
Recall that we assume $q^{-1} = (c + 1)p + 1$, which implies $(1 - q)/q = q^{-1} - 1 = (c + 1)p$. Plugging this equality into the previous one completes the proof that the expectations of the last two terms on the rightmost hand side of Inequality~\eqref{eq:cup_bound} are equal.
\end{proof}

Proving our result for monotone functions (Theorem~\ref{thm:monotone}) is now straightforward.
\begin{proof}[Proof of Theorem~\ref{thm:monotone}]
By plugging $c = 1$ and $q^{-1} = 2p + 1$ into Algorithm~\ref{alg:actual}, we get an algorithm which uses $O(k)$ memory and $O(km/p)$ oracle queries by Observation~\ref{obs:complexity}. Additionally, by Theorem~\ref{thm:raw}, this algorithm obeys
\[
	\bE[f(S_n)]
	\geq
	\frac{c}{(1 + c)^2p} \cdot \bE[f(A \cup OPT)]
	=
	\frac{1}{4p} \cdot \bE[f(A \cup OPT)]
	\geq
	\frac{1}{4p} \cdot f(OPT)
	\enspace,
\]
where the second inequality follows from the monotonicity of $f$. Thus, the approximation ratio of the algorithm we got is at most $4p$.
\end{proof}

Proving our result for non-monotone functions is a bit more involved. First, we need the following known lemma.
\begin{lemma}[Lemma 2.2 of~\citep{BFNS14}] \label{lem:probability_bound}
Let $g\colon 2^\cN \to \nnR$ be a non-negative submodular function, and let $B$ be a random subset of $\cN$ containing every element of $\cN$ with probability at most $q$ (not necessarily independently), then
$\bE[g(B)] \geq (1 - q) \cdot g(\varnothing)$.
\end{lemma}

The proof of Theorem~\ref{thm:non_monotone} is now very similar to the above presented proof of Theorem~\ref{thm:monotone}, except that slightly different values for $c$ and $q$ are used, and in addition, Lemma~\ref{lem:probability_bound} is now used to lower bound $\bE[f(A \cup OPT)]$ instead of the monotonicity of the objective that was used for that purpose in the proof of Theorem~\ref{thm:monotone}. A more detailed presentation of this proof is given below.

\begin{proof}[Proof of Theorem~\ref{thm:non_monotone}]
	By plugging $c = \sqrt{1 + 1/p}$ and $q^{-1} = p + \sqrt{p(p+1)} + 1$ into Algorithm~\ref{alg:actual}, we get an algorithm which uses $O(k)$ memory and $O(km/p)$ oracle queries by Observation~\ref{obs:complexity}. Additionally, by Theorem~\ref{thm:raw}, this algorithm obeys
	\[
	\bE[f(S_n)]
	\geq
	\frac{c}{(1 + c)^2p} \cdot \bE[f(A \cup OPT)]
	\enspace.
	\]
	Let us now define $g\colon 2^\cN \to \nnR$ to be the function $g(S) = f(S \cup OPT)$. Note that $g$ is non-negative and submodular. Thus, by Lemma~\ref{lem:probability_bound} and the fact that $A$ contains every element with probability at most $q$ (because Algorithm~\ref{alg:actual} accepts an element into its solution with at most this probability), we get
	\begin{align*}
	\bE[f(A \cup OPT)]
	={} &
	\bE[g(A)]
	\geq
	(1 - q) \cdot g(\varnothing)\\
	={} &
	(1 - q) \cdot f(OPT)
	=
	\frac{p + \sqrt{p(p+1)}}{p + \sqrt{p(p+1)} + 1} \cdot f(OPT)\\
	={} &
	\frac{p + \sqrt{p(p+1)}}{\sqrt{1 + 1/p} \cdot (p + \sqrt{p(p+1)})} \cdot f(OPT)
	=
	\frac{1}{c} \cdot f(OPT)
	\enspace.
	\end{align*}
	Combining the two above inequalities, we get
	\begin{align*}
	\bE[f(S_n)]
	\geq{} &
	\frac{1}{(1 + c)^2p} \cdot f(OPT)
	=
	\frac{1}{(2 + 2\sqrt{1 + 1/p} + 1/p)p} \cdot f(OPT)\\
	={} &
	\frac{1}{2p + 2\sqrt{p(p + 1)} + 1} \cdot f(OPT)
	\end{align*}
	Thus, the approximation ratio of the algorithm we got is at most $2p + 2\sqrt{p(p + 1)} + 1$.
\end{proof}

\input{./tex/Mapping}

%% file: tex/Mapping.tex
\subsection{Proof of Proposition~\ref{prop:mapping}} 
\label{ssc:mapping}
In this section we prove Propsition~\ref{prop:mapping}. Let us first restate the proposition itself.
\begin{repproposition}{prop:mapping}
\propMapping
\end{repproposition}

We begin the proof of Proposition~\ref{prop:mapping} by constructing $m$ graphs, one for every one of the matroids defining $\cM$. For every $1 \leq \ell \leq m$, the graph $G_\ell$ contains two types of vertices: its internal vertices are the elements of $A \cap \cN_\ell$, and its external vertices are the elements of $\{u_i \in \cN_\ell \setminus (R \cup A) \mid (S_{i-1} + u_i) \cap \cN_\ell \not \in \cI_\ell\}$. Informally, the external elements of $G_\ell$ are the elements of $\cN_\ell$ which were rejected upon arrival by Algorithm~\ref{alg:analysis} and the matroid $\cM_\ell = (\cN_\ell, \cI_\ell)$ can be (partially) blamed for this rejection. The arcs of $G_\ell$ are created using the following iterative process that creates some arcs of $G_\ell$ in response to every arriving element. For every $1 \leq i \leq n$, consider the element $x_\ell$ selected by the execution of {\ExchangeAlg} on the element $u_i$ and the set $S_{i-1}$. From this point on we denote this element by $x_{i, \ell}$. If no $x_{i, \ell}$ element was selected by the above execution of {\ExchangeAlg}, or $u_i \in R$, then no $G_\ell$ arcs are created in response to $u_i$. Otherwise, let $C_{i, \ell}$ be the single cycle of the matroid $\cM_\ell$ in the set $(S_{i-1} + u_i) \cap \cN_\ell$---there is exactly one cycle of $\cM_\ell$ in this set because $S_{i-1}$ is independent, but $(S_{i-1} + u_i) \cap \cN_\ell$ is not independent in $\cM_\ell$. One can observe that $C_{i, \ell} - u_i$ is equal to the set $X_\ell$ in the above mentioned execution of {\ExchangeAlg}, and thus, $x_{i, \ell} \in C_{i, \ell}$. We now denote by $u'_{i, \ell}$ the vertex out of $\{u_i, x_{i, \ell}\}$ that does not belong to $S_i$---notice that there is exactly one such vertex since $x_{i, \ell} \in U_i$, which implies that it appears in $S_i$ if $S_i = S_{i-1}$ and does not appear in $S_i$ if $S_i = S_{i-1} \setminus U_i + u_i$. Regardless of the node chosen as $u'_i$, the arcs of $G_\ell$ created in response to $u_i$ are all the possible arcs from $u'_{i, \ell}$ to the other vertices of $C_{i, \ell}$. Observe that these are valid arcs for $G_\ell$ in the sense that their end points (\ie, the elements of $C_{i, \ell}$) are all vertices of $G_\ell$---for the elements of $C_{i, \ell} - u_i$ this is true since $C_{i, \ell} - u_i \subseteq S_{i-1} \cap \cN_\ell \subseteq A \cap \cN_\ell$, and for the element $u_i$ this is true since the existence of $x_{i, \ell}$ implies $(S_{i-1} + u_i) \cap \cN_\ell \not \in \cI_\ell$.

Some properties of $G_\ell$ are given by the following observation. Given a graph $G$ and a vertex $u$, we denote by $\delta^+_G(u)$ the set of vertices to which there is a direct arc from $u$ in $G$.
\newcommand{\obsGraphProperties}{%
For every $1 \leq \ell \leq m$,
\begin{itemize}
	\item every non-sink vertex $u$ of $G_\ell$ is spanned by the set $\delta^+_{G_\ell}(u)$.
	\item for every two indexes $1 \leq i, j \leq n$, if $u'_{i, \ell}$ and $u'_{j, \ell}$ both exist and $i \neq j$, then $u'_{i, \ell} \neq u'_{j, \ell}$.
	\item $G_\ell$ is a directed acyclic graph.
\end{itemize}}
\begin{observation} \label{obs:graph_properties}
\obsGraphProperties
\end{observation}
\begin{proof}
Consider an arbitrary non-sink node $u$ of $G_\ell$. Since there are arcs leaving $u$, $u$ must be equal to $u'_{i, \ell}$ for some $1 \leq i \leq n$. This implies that $u$ belongs to the cycle $C_{i, \ell}$, and that there are arcs from $u$ to every other vertex of $C_{i, \ell}$. Thus, $u$ is spanned by the vertices of $\delta^+_{G_\ell}(u) \supseteq C_{i, \ell} - u$ because the fact that $C_{i, \ell}$ is a cycle containing $u$ implies that $C_{i, \ell} - u$ spans $u$. This completes the proof of the first part of the observation.

Let us prove now a very useful technical claim. Consider an index $1 \leq i \leq n$ such that $u'_i$ exists, and let $j$ be an arbitrary value $i < j \leq n$. We will prove that $u'_i$ does not belong to $C_{j, \ell}$. By definition, $u'_i$ is either $u_i$ or the vertex $x_{i, \ell}$ that belongs to $S_{i - 1}$, and thus, arrived before $u_i$ and is not equal to $u_j$; hence, in neither case $u'_i \neq u_j$. Moreover, combining the fact that $u'_i$ is either $u_i$ or arrived before $u_i$ and the observation that $u'_i$ is never a part of $S_i$, we get that $u'_i$ cannot belong to $S_j \supseteq C_{j, \ell} - u_j$, which implies the claim together with out previous observation that $u'_i \neq u_j$.

The technical claim that we proved above implies the second part of the lemma, namely that for every two indexes $1 \leq i, j \leq n$, if $u'_{i, \ell}$ and $u'_{j, \ell}$ both exist and $i \neq j$, then $u'_{i, \ell} \neq u'_{j, \ell}$. To see why that is the case, assume without loss of generality $i < j$. Then, the above technical claim implies that $u'_{i, \ell} \not \in C_{j, \ell}$, which implies $u'_{i, \ell} \neq u'_{j, \ell}$ because $u'_{j, \ell} \in C_{j, \ell}$.

At this point, let us assume towards a contradiction that the third part of the observation is not true, \ie, that there exists a cycle $L$ in $G_\ell$. Since every vertex of $L$ has a non-zero out degree, every such vertex must be equal to $u'_{i, \ell}$ for some $1 \leq i \leq n$. Thus, there must be indexes $1 \leq i_1 < i_2 \leq n$ such that $L$ contains an arc from $u'_{i_2,\ell}$ to $u'_{i_1, \ell}$. Since we already proved that $u'_{i_2,\ell}$ cannot be equal to $u'_{j, \ell}$ for any $j \neq i_2$, the arc from $u'_{i_2,\ell}$ to $u'_{i_1, \ell}$ must have been created in response to $u_{i_2}$, hence, $u'_{i_1, \ell} \in C_{i_2, \ell}$, which contradicts the technical claim we have proved.
%
%
\end{proof}

One consequence of the properties of $G_\ell$ proved by the last observation is given by the following lemma. A slightly weaker version of this lemma was proved implicitly by~\cite{V11}, and was stated as an explicit lemma by~\cite{CGQ15}.

\newcommand{\lemFunctionGraph}{%
Consider an arbitrary directed acyclic graph $G = (V, E)$ whose vertices are elements of some matroid $\cM'$. If every non-sink vertex $u$ of $G$ is spanned by $\delta^+_G(u)$ in $\cM'$, then for every set $S$ of vertices of $G$ which is independent in $\cM'$ there must exist an injective function $\psi_S$ such that, for every vertex $u \in S$, $\psi_S(u)$ is a sink of $G$ which is reachable from $u$.}
\begin{lemma} \label{lem:function_graph}
\lemFunctionGraph
\end{lemma}
\begin{proof}
Let us define the \emph{width} of a set $S$ of vertices of $G$ as the number of arcs that appear on some path starting at a vertex of $S$ (more formally, the width of $S$ is the size of the set $\{e \in E \mid \text{there is a path in $G$ that starts in a vertex of $S$ and includes $e$}\}$). We prove the lemma by induction of the width of $S$. First, consider the case that $S$ is of width $0$. In this case, the vertices of $S$ cannot have any outgoing arcs because such arcs would have contributed to the width of $S$, and thus, they are all sinks of $G$. Thus, the lemma holds for the trivial function $\psi_S$ mapping every element of $S$ to itself. Assume now that the width $w$ of $S$ is larger than $0$, and assume that the lemma holds for every set of width smaller than $w$. Let $u$ be a non-sink vertex of $S$ such that there is no path in $G$ from any other vertex of $S$ to $u$. Notice that such a vertex must exist since $G$ is acyclic. By the assumption of the lemma, $\delta^+(u)$ spans $u$. In contrast, since $S$ is independent, $S - u$ does not span $u$, and thus, there must exist an element $v \in \delta^+(u) \setminus S$ such that the set $S' = S - u + v$ is independent.

Let us explain why the width of $S'$ must be strictly smaller than the width of $S$. First, consider an arbitrary arc $e$ which is on a path starting at a vertex $u' \in S'$. If $u' \in S$, then $e$ is also on a path starting in a vertex of $S$. On the other hand, if $u' \not \in S$, then $u'$ must be the vertex $v$. Thus, $e$ must be on a path $P$ starting in $v$. Adding $uv$ to the beginning of the path $P$, we get a path from $u$ which includes $e$. Hence, in conclusion, we have got that every arc $e$ which appears on a path starting in a vertex of $S'$  (and thus, contributes to the width of $S'$) also appears on a path starting in a vertex of $S$ (and thus, also contributes to the width of $S$); which implies that the width of $S'$ is not larger than the width of $S$. To see that the width of $S'$ is actually strictly smaller than the width of $S$, it only remains to find an arc which contributes to the width of $S$, but not to the width of $S'$. Towards this goal, consider the arc $uv$. Since $u$ is a vertex of $S$, the arc $uv$ must be on some path starting in $u$ (for example, the path including only this arc), and thus, contributes to the width of $S$. Assume now towards a contradiction that $uv$ contributes also to the width of $S'$, \ie, that there is a path $P$ starting at a vertex $w \in S'$ which includes $uv$. If $w = v$, then this leads to a contradiction since it implies the existence of a cycle in $G$. On the other hand, if $w \neq v$, then this implies a path in $G$ from a vertex $w \neq u$ of $S$ to $u$, which contradicts the definition of $u$. This completes the proof that the width of $S'$ is strictly smaller than the width of $S$.

Using the induction hypothesis, we now get that there exists an injective function $\psi_{S'}$ mapping every vertex of $S'$ to a sink of $G$. Using $\psi_{S'}$, we can define $\psi_S$ as follows. For every $w \in S$,
\[
	\psi_S(w) =
	\begin{cases}
		\psi_{S'}(v) & \text{if $w = u$} \enspace,\\
		\psi_{S'}(w) & \text{otherwise} \enspace.
	\end{cases}
\]
Since $u$ appears in $S$ but not in $S'$, and $v$ appears in $S'$ but not in $S$, the injectiveness of $\psi_S$ follows from the injectiveness of $\psi_{S'}$. Moreover, $\psi_S$ clearly maps every vertex of $S$ to a sink of $G$ since $\psi_{S'}$ maps every vertex of $S'$ to such a sink. Finally, one can observe that $\psi_S(w)$ is reachable from $w$ for every $w \in S$ because $\psi_S(u) = \psi_{S'}(v)$ is reachable from $v$ by the definition of $\psi_{S'}$, and thus, also from $u$ due to the existence of the arc $uv$.
\end{proof}

For every $1 \leq \ell \leq m$, let $T_\ell$ be the set of elements of $T$ that appear as vertices of $G_\ell$. Since $T$ is independent and $T_\ell$ contains only elements of $\cN_\ell$, Observation~\ref{obs:graph_properties} and Lemma~\ref{lem:function_graph} imply together the existence of an injective function $\psi_{T_\ell}$ mapping the elements of $T_\ell$ to sink vertices of $G_\ell$. We can now define the function $\phi_T$ promised by Proposition~\ref{prop:mapping}. For every element $u \in T$, the function $\phi_T$ maps $u$ to the multi-set $\{\psi_{T_\ell}(u) \mid 1 \leq \ell \leq m \text{ and } u \in T_\ell\}$, where we assume that repetitions are kept when the expression $\psi_{T_\ell}(u)$ evaluates to the same element for different choices of $\ell$. Let us explain why the elements in the multi-sets produced by $\phi_T$ are indeed all elements of $A$, as is required by the proposition. Consider an element $u_i \not \in A$, and let us show that it does not appear in the range of $\psi_{T_\ell}$ for any $1 \leq \ell \leq m$. If $u_i$ does not appear as a vertex in $G_\ell$, then this is obvious. Otherwise, the fact that $u_i \not \in A$ implies $u'_{i, \ell} = u_i$, and thus, the arcs of $G_\ell$ created in response to $u_i$ are arcs leaving $u_i$, which implies that $u_i$ is not a sink of $G_\ell$, and hence, does not appear in the range of $\psi_{T_\ell}$. 

Recall that every element $u \in \cN$ belongs to at most $p$ out of the ground sets $\cN_1, \cN_2, \dotsc,\allowbreak \cN_m$, and thus, is a vertex in at most $p$ out of the graphs $G_1, G_2, \dotsc, G_m$. Since $\psi_{T_\ell}$ maps every element to vertexes of $G_\ell$, this implies that $u$ is in the range of at most $p$ out of the functions $\psi_{T_1}, \psi_{T_2}, \dotsc, \psi_{T_m}$. Moreover, since these functions are injective, every one of these functions that have $u$ in its range maps at most one element to $u$. Thus, the multi-sets produced by $\phi_T$ contain $u$ at most $p$ times. Since this is true for every element of $\cN$, it is true in particular for the elements of $S_n$, which is the first property of $\phi_T$ that we needed to prove.

Consider now an element $u \in A \setminus S_n$. Our next objective is to prove that $u$ appears at most $p - 1$ times in the multi-sets produced by $\phi_T$, which is the second property of $\phi_T$ that we need to prove. Above, we proved that $u$ appears at most $p$ times in these multi-sets by arguing that every such appearance must be due to a function $\psi_{T_\ell}$ that has $u$ in its range, and that the function $\psi_{T_\ell}$ can have this property only for the $p$ values of $\ell$ for which $u \in \cN_\ell$. Thus, to prove that $u$ in fact appears only $p - 1$ times in the multi-sets produced by $\phi_T$, it is enough to argue that there exists a value $\ell$ such that $e \in \cN_\ell$, but $\psi_{T_\ell}$ does not have $u$ in its range. Let us prove that this follows from the membership of $u$ in $A \setminus S_n$. Since $u$ was removed from the solution of Algorithm~\ref{alg:analysis} at some point, there must be some index $1 \leq i \leq n$ such that both $u \in U_i$ and $u_i$ was added to the solution of Algorithm~\ref{alg:analysis}. Since $u \in U_i$, there must be a value $1 \leq \ell \leq m$ such that $u = x_{i, \ell}$, and since $u_i$ was added to the solution of Algorithm~\ref{alg:analysis}, $u'_{i, \ell} = x_{i, \ell}$. These equalities imply together that there are arcs leaving $u$ in $G_\ell$ (which were created in response to $u_i$). Thus, the function $\psi_{T_\ell}$ does not map any element to $u$ because $u$ is not a sink of $G_\ell$, despite the fact that $u \in \cN_\ell$.

To prove the other guaranteed properties of $\phi_T$, we need the following lemma.
\newcommand{\lemStepInternal}{%
Consider two vertices $u_i$ and $u_j$ such that $u_j$ is reachable from $u_i$ in $G_\ell$. If $u_i \in A$, then $f(u_i : S_{d(i) - 1}) \leq f(u_j : S_{d(j) - 1})$, otherwise, $f(x_{i, \ell} : S_{i-1}) \leq f(u_j : S_{d(j)-1})$.}
\begin{lemma} \label{lem:step_internal}
\lemStepInternal
\end{lemma}
\begin{proof}
We begin by proving the special case of the lemma in which $u_i \in A$ (\ie, is an internal vertex of $G_\ell)$ and there is a direct arc from $u_i$ to $u_j$. 
The existence of this arc implies that there is some value $1 \leq h \leq n$ such that $u'_{h, \ell} = u_i$ and $u_j \in C_{h, \ell}$. Since $u_i$ is internal, it cannot be equal be to $u_h$ because this would have implied that $u_h$ was rejected immediately by Algorithm~\ref{alg:analysis}, and is thus, not internal. Thus, $u_i = x_{h, \ell}$. Recall now that $C_{h, \ell} - u_h$ is equal to the set $X_\ell$ chosen by ${\ExchangeAlg}$ when it is executed with the element $u_h$ and the set $S_{h-1}$. Thus, the fact that $u_i = x_{h, \ell}$ and the way $x_{h, \ell}$ is chosen out of $X_\ell$ implies that whenever $u_j \neq u_h$ we have
\[
	f(u_i : S_{d(i) - 1})
	=
	f(u_i : S_{h-1})
	\leq
	f(u_j : S_{h-1})
	\leq
	f(u_j : S_{d(j)-1})
	\enspace,
\]
where the equality holds since $u'_{h, \ell} = u_i$ implies $d(i) = h$ and the last inequality holds since $f(u_j : S_{r - 1})$ is a non-decreasing function of $r$ when $r \geq j$ and the membership of $u_j$ in $C_{h, \ell}$ implies $j \leq h \leq d(j)$.

It remains to consider the case $u_j = u_h$. In this case, the fact that $u_j = u_h$ is accepted into the solution of Algorithm~\ref{alg:analysis} implies
\begin{align*}
	f(u_j : S_{d(j)-1})
	\geq{} &
	f(u_j : S_{j-1})
	=
	f(u_j \mid S_{j-1} \cap \{u_1, u_2,\dotsc, u_{j-1}\})
	=
	f(u_j \mid S_{j-1})\\
	={} &
	f(u_h \mid S_{h-1})
	\geq
	(1 + c) \cdot f(U_h : S_{h - 1})
	\geq
	f(U_h : S_{h - 1})\\
	\geq{} &
	f(x_{h, \ell} : S_{h-1})
	=
	f(u_i : S_{h-1})
	=
	f(u_i : S_{d(i) - 1})
	\enspace,
\end{align*}
where the first inequality holds since $d(j) \geq j$ by definition, the last equality holds since $u'_{h, \ell} = u_i$ implies $d(i) = h$ and the two last inequalities follow from the fact that the elements of $U_h \subseteq A$ do not belong to $Z$ by Observation~\ref{obs:set_properties}, which implies (again, by Observation~\ref{obs:set_properties}) that $f(u : S_{h-1}) \geq 0$ for every $u \in U_h$. This completes the proof of the lemma for the special case that $u_i \in A$ and there is a direct arc from $u_i$ to $u_j$.

Next, we prove that no arc of $G_\ell$ goes from an internal vertex to an external one. Assume this is not the case, and that there exists an arc $uv$ of $G_\ell$ from an internal vertex $u$ to an external vertex $v$. By definition, there must be a value $1 \leq h \leq n$ such that $v$ belongs to the cycle $C_{h, \ell}$ and $u'_{h, \ell} = u$. The fact that $u$ is an internal vertex implies that $u_h$ must have been accepted by Algorithm~\ref{alg:analysis} upon arrival becuase otherwise we would have gotten $u = u'_{h, \ell} = u_{h}$, which implies that $u$ is external, and thus, leads to a contradiction. Consequently, we get $C_{h, \ell} \subseteq A$ because every element of $C_{h, \ell}$ must either be $u_{h}$ or belong to $S_{h - 1}$. In particular, $v \in A$, which contradicts our assumption that $v$ is an external vertex.

We are now ready to prove the lemma for the case $u_i \in A$ (even when there is no direct arc in $G_\ell$ from $u_i$ to $u_j$). Consider some path $P$ from $u_i$ to $u_j$, and let us denote the vertices of this path by $u_{r_0}, u_{r_1}, \dotsc, u_{r_{|P|}}$. Since $u_i$ is an internal vertex of $G_\ell$ and we already proved that no arc of $G_\ell$ goes from an internal vertex to an external one, all the vertices of $P$ must be internal. Thus, by applying the special case of the lemma that we have already proved to every pair of adjacent vertices along the path $P$, we get that the expression $f(u_{r_k} : S_{d(r_k) - 1})$ is a non-decreasing function of $k$, and in particular,
\[
	f(u_i : S_{d(i) - 1})
	=
	f(u_{r_0} : S_{d(r_0) - 1})
	\leq
	f(u_{r_k} : S_{d(r_k) - 1})
	=
	f(u_j : S_{d(j) - 1})
	\enspace.
\]

It remains to prove the lemma for the case $u_i \not \in A$. Let $u_h$ denote the first vertex on some path from $u_i$ to $u_j$ in $G_\ell$. Since $u_i \not \in A$, we get that $u'_{i, \ell} = u_i$, which implies that the arcs of $G_\ell$ that were created in response to $u_i$ go from $u_i$ to the vertices of $C_{i, \ell} - u_i$. Since Observation~\ref{obs:graph_properties} gurantees that $u_i = u'_{i, \ell} \neq u'_{j, \ell}$ for every value $1 \leq j \leq n$ which is different from $i$, there cannot be any other arcs in $G_\ell$ leaving $u_i$, and thus, the existence of an arc from $u_i$ to $u_h$ implies $u_h \in C_{i, \ell} - u_i$. Recall now that $C_{i, \ell} - u_i$ is equal to the set $X_\ell$ in the execution of {\ExchangeAlg} corresponding to the element $u_i$ and the set $S_{i-1}$, and thus, by the definition of $x_{i, \ell}$, $f(x_{i, \ell} : S_{i-1}) \leq f(u_h : S_{i-1})$. Additionally, as an element of $C_{i, \ell} - u_i$, $u_h$ must be a member of $S_{i-1} \subseteq A$, and thus, by the part of the lemma we have already proved, we get $f(u_h : S_{d(h)-1}) \leq f(u_j : S_{d(j) - 1})$ because $u_j$ is reachable from $u_h$. Combining the two inequalities we have proved, we get
\[
	f(x_{i, \ell} : S_{i-1})
	\leq
	f(u_h : S_{i-1})
	\leq
	f(u_h : S_{d(h)-1})
	\leq
	f(u_j : S_{d(j) - 1})
	\enspace,
\]
where the second inequality holds since the fact that $u_h \in C_{i, \ell} - u_i \subseteq S_{i-1}$ implies $d(h) \geq i$.
\end{proof}

Consider now an arbitrary element $u_i \in T \setminus A$. Let us denote by $u_{r_\ell}$ the element $u_{r_\ell} = \psi_{T_\ell}(u_i)$ if it exists, and recall that this element is reachable from $u_i$ in $G_\ell$. Thus, the fact that $u_i$ is not in $A$ implies
\begin{align*}
	f(u_i \mid S_{i-1})
	\leq{} &
	(1 + c) \cdot \sum_{u \in U_i} f(u : S_{i-1})
	=
	(1 + c) \cdot \mspace{-18mu} \sum_{\substack{1 \leq \ell \leq m\\(S_{i-1} + u_i) \cap \cN_\ell \not \in \cI_\ell}} \mspace{-36mu} f(x_{i, \ell} : S_{i-1})\\
	\leq{} &
	(1 + c) \cdot \mspace{-18mu} \sum_{\substack{1 \leq \ell \leq m\\(S_{i-1} + u_i) \cap \cN_\ell \not \in \cI_\ell}} \mspace{-36mu} f(u_{r_\ell} : S_{d(r_\ell)})
	=
	(1 + c) \cdot \mspace{-18mu} \sum_{u_j \in \phi_T(u_i)} f(u_j : S_{d(j)})
	\enspace,
\end{align*}
where the inequality follows from Lemma~\ref{lem:step_internal} and the last equality holds since the values of $\ell$ for which $(S_{i-1} + u_i) \cap \cN_\ell \not \in \cI_\ell$ are exactly the values for which $u_i \in T_\ell$, and thus, they are all also exactly the values for which the multi-set $\phi_T(u_i)$ includes the value of $\psi_{T_\ell}(u_i)$. This completes the proof of the third property of $\phi_T$ that we need to prove.

Finally, consider an arbitrary element $u_i \in A \cap T$. Every element $u_j \in \phi_T(u_i)$ can be reached from $u_i$ in some graph $G_\ell$, and thus, by Lemma~\ref{lem:step_internal},
\begin{align*}
	f(u_i \mid S_{i - 1})
	={} &
	f(u_i \mid S_{i - 1} \cap \{u_1, u_2, \dotsc, u_{i-1}\})
	=
	f(u_i : S_{i-1})\\
	\leq{} &
	f(u_i : S_{d(i)-1})
	\leq
	f(u_j : S_{d(j)-1})
	\enspace,
\end{align*}
where the first inequality holds since $d(i) \geq i$ by definition and $f(u_i : S_{r-1})$ is a non-decreasing function of $r$ for $r \geq i$. Additionally, we observe that $u_i$, as an element of $T \cap A$, belongs to $T_\ell$ for every value $1 \leq \ell \leq m$ for which $u_i \in \cN_\ell$, and thus, the size of the multi-set $\phi_T(u_i)$ is equal to the number of ground sets out of $\cN_1, \cN_2, \dotsc, \cN_m$ that include $u_i$. Since we assume by Reduction~\ref{red:exact} that every element belongs to exactly $p$ out of these ground sets, we get that the multi-set $\phi_T(u_i)$ contains exactly $p$ elements (including repetitions), which completes the proof of Proposition~\ref{prop:mapping}.

%% file: tex/Experiment.tex
\section{Experiment}

In this section, we evaluate the performance of our algorithm (\AlgSampling) on a video summarization task. 
We compare our algorithm with \AlgSeq \citep{GCGS14}\footnote{\url{https://github.com/pujols/Video-summarization}} and  \AlgLocal \citep{MJK17}.\footnote{\url{https://github.com/baharanm/non-mon-stream}}
 For our experiments, we use the Open Video Project (OVP) and the YouTube datasets, which have 50 and 39 videos, respectively \citep{DLDA11}. 

Determinantal point process (DPP) is a powerful method to capture diversity in datasets \citep{M75,kulesza2012determinantal}. 
Let $\cN = \{1,2, \cdots,  n\}$ be a ground set of $n$ items.
A DPP defines a probability distribution over all subsets of $\cN$, and a random variable $Y$ distributed according to this distribution obeys $\Pr[Y = S] = \frac{\det (L_S)}{\det(I + L)}$ for every set $S \subseteq \cN$,
where $L$ is a positive semidefinite kernel matrix, $L_S$  is the principal sub-matrix of $L$ indexed by $S$ and $I$ is the $n \times n$ identity matrix. 
The most divers subset of $\cN$ is the one with the maximum probability in this distribution. Unfortunately, finding this set is NP-hard \citep{KLQ95}, but the function $f(S) = \log \det (L_S)$ is a non-monotone submodular function \citep{kulesza2012determinantal}.

We follow the experimental setup of  \citep{GCGS14} for extracting frames from videos, finding a linear kernel matrix  $L$ and evaluating the quality of produced summaries based on their F-score.
\citet{GCGS14} define a sequential DPP, where each video sequence is partitioned into disjoint segments of equal sizes. For selecting a subset $S_t$ from each segment $t$ (i.e., set $\cP_t$), a DPP is defined on the union of the frames in this segment and the selected frames $S_{t-1}$  from the previous segment. 
Therefore, the conditional distribution of $S_t$ is given by, $\Pr[S_t|S_{t-1}] =  \frac{\det (L_{S_t \cup S_{t-1}})}{\det(I_t + L)},$ where $L$ is the kernel matrix define over $\cP_t \cup S_{t-1}$, and $I_t$ is a diagonal matrix of the same size as $\cP_t \cup S_{t-1} $ in which the elements corresponding to $S_{t-1}$ are zeros and
the elements corresponding to $\cP_t$ are $1$.
For the detailed explanation, please refer to \citep{GCGS14}.
In our experiments, we focus on maximizing the non-monotone submodular function $f(S_t) = \log \det (L_{S_t \cup S_{t-1}})$. 
We would like to point out that this function can take negative values, which is slightly different from the non-negativity condition we need for our theoretical guarantees.

We first compare the objective values  (F-scores) of the algorithms \AlgSampling and \AlgLocal for different segment sizes over YouTube and OVP datasets. In each experiment, the values are normalized to the F-score of summaries generated by \AlgSeq. 
In Figures~\ref{fig:objective-value}(a) and \ref{fig:objective-value}(b), we observe that both algorithms produce summaries with very high qualities. Figure~\ref{fig:video-60} shows the summary produced by our algorithm for OVP video number 60.
\citet{MJK17} showed that their algorithm (\AlgLocal) runs three orders of magnitude faster than \AlgSeq \citep{GCGS14}.
 In our experiments (see Figure~\ref{fig:objective-value}(c)), we observed that \AlgSampling is 40 and 50 times faster than \AlgLocal  for the YouTube and OVP datasets, respectively. 
 Note that  for different segment sizes the number of frames remains constant; therefore, the time complexities for both \AlgSampling and \AlgLocal do not change.

\begin{figure}[t!] 
	\centering
	\includegraphics[height=1.62in]{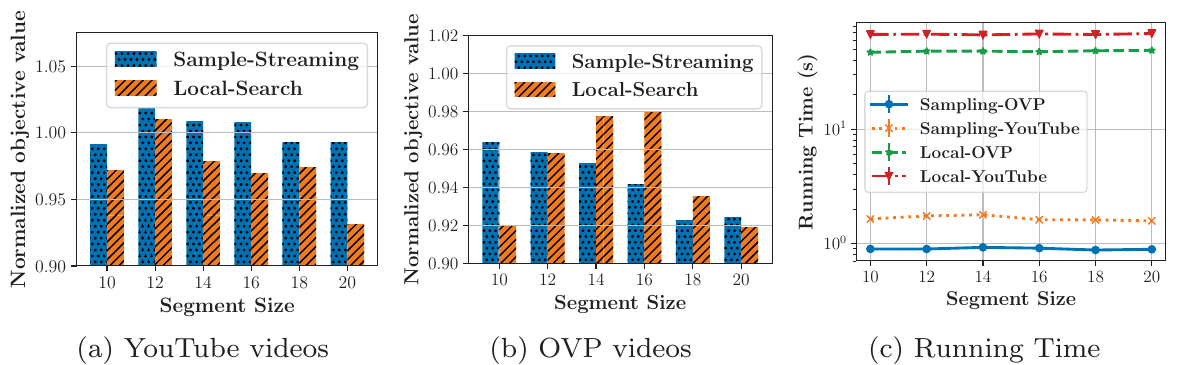}
	\caption{Comparing the normalized objective value and running time
of \AlgSampling and \AlgLocal for different segment sizes.}\label{fig:objective-value}
\end{figure}
\begin{figure}[t!]
	\centering

	\includegraphics[height=.43in]{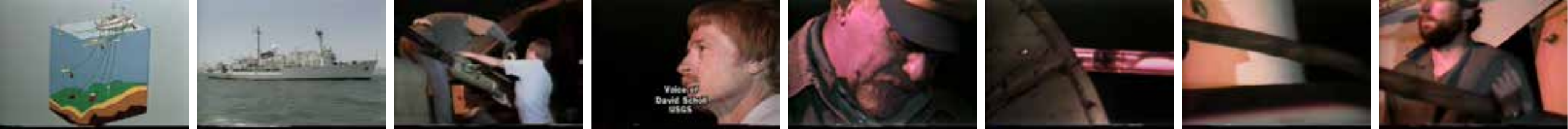}
	\caption{Summary generated by \AlgSampling for OVP video number 60.}
	\label{fig:video-60}
\end{figure}

\begin{figure}[t!]
	\centering
	\includegraphics[height=1.4in]{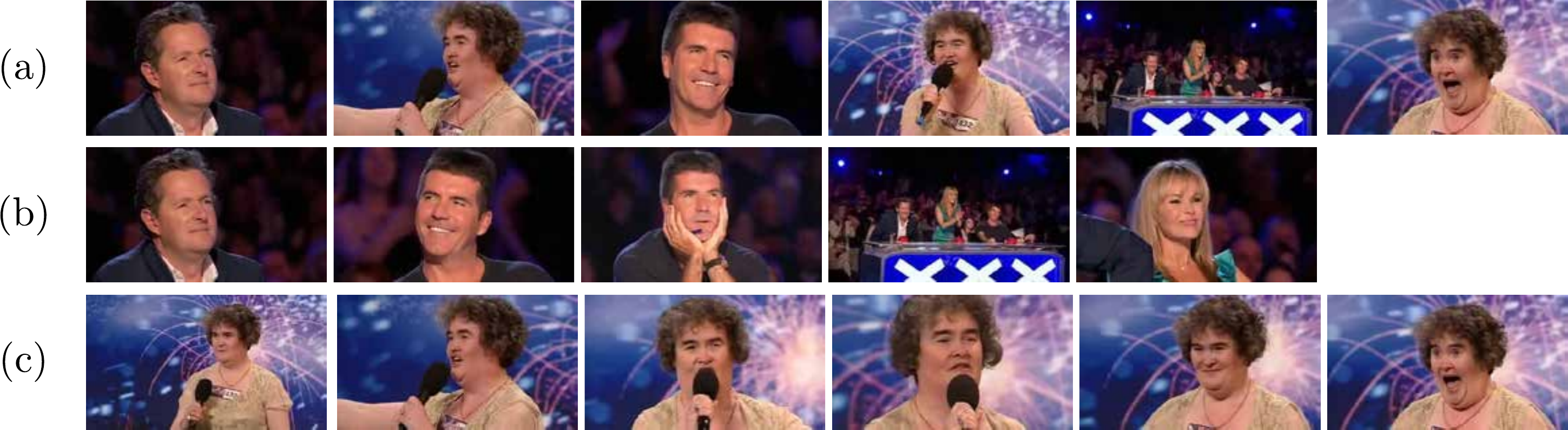}
	\caption{Summaries generated by \AlgSampling for YouTube video number 106: (a)  a $6$-matchoid constraint, (b) a $3$-matchoid constraint and (c) a partition matroid constraint.}
	\label{fig:constraint}
\end{figure}

In the last experiment, we study the effect of imposing different constraints on video summarization task for YouTube video number 106, which is a part of the America's Got Talent series.
In the first set of constraints, we consider 6 (for 6 different faces in the frames) partition matroids to limit the number of frames containing each face $i$, i.e., a $6$-matchoid constraint\footnote{Note that a frame may contain more than one face.}
$\mathcal{I} = \{ S \subseteq \cN : |S \cap \cN_i| \leq k_i\}$, where  $\cN_i \subseteq \cN$ is the set of frames containing face $i$ for $1 \leq i \leq 6$. For all the $i$ values, we set $k_i = 3$.
In this experiment, we use the same methods as described by \citet{MJK17} for face recognition. Figure~\ref{fig:constraint}(a) shows the summary produced for this task. The second set of constraints is a $3$-matchoid, where matroids limit the number of 
frames containing each one of the three judges. 
The summary for this constraint is shown in  Figure~\ref{fig:constraint}(b).
Finally, Figure~\ref{fig:constraint}(c) shows a summary with a single partition matroid constraint on the singer.

%% file: tex/Conclusion.tex
\section{Conclusion}  \label{sec:conlusion}
We developed a streaming algorithm for  submodular maximization by carefully subsampling elements of the data stream.
Our algorithm provides the best of three worlds: (i) the tightest approximation guarantees in various settings, including $p$-matchoid and matroid constraints for non-monotone submodular functions, (ii) minimum memory requirement, and (iii) fewest queries per element.
We also experimentally studied the effectiveness of our algorithm in a video summarization task.